\documentclass[letterpaper]{article} 
\usepackage{aaai24}  
\usepackage{times}  
\usepackage{helvet}  
\usepackage{courier}  
\usepackage[hyphens]{url}  
\usepackage{graphicx} 
\usepackage{natbib}  
\usepackage{caption} 
\usepackage{algorithm}
\usepackage{algorithmic}

\usepackage{newfloat}
\usepackage{listings}
\DeclareCaptionStyle{ruled}{labelfont=normalfont,labelsep=colon,strut=off} 
\floatstyle{ruled}
\newfloat{listing}{tb}{lst}{}
\floatname{listing}{Listing}
\usepackage{mathtools}
\usepackage{amsmath}
\usepackage{amssymb}
\usepackage{mathrsfs}  

\newtheorem{theorem}{Theorem}

\newtheorem{corollary}{Corollary}

\newtheorem{definition}{Definition}

\newtheorem{proposition}{Proposition}

\newenvironment{proof}[1][Proof]{\textbf{#1.} }{\ \rule{0.5em}{0.5em} \vspace{1ex}}

\def\real{\mathbb{R}}
\newcommand{\Rexp}{\real_{\text{exp.}}}
\newcommand{\Ran}{\real_{\text{an.}}}
\newcommand{\Ranexp}{\real_{\text{an.,exp.}}}
\newcommand{\Ralg}{\real_{\text{semialg.}}}
\newcommand{\BR}{\mathbb{R}}
\newcommand{\BN}{\mathbb{N}}
\newcommand{\CS}{\mathcal{S}}
\newcommand{\rd}{\mathrm{d}}

\title{Taming Binarized Neural Networks \\ and Mixed-Integer Programs}
\author{
    Johannes Aspman\textsuperscript{\rm 1}, Georgios Korpas\textsuperscript{\rm 1,2} and
    Jakub Marecek\textsuperscript{\rm 1}
}
\affiliations{
    \textsuperscript{\rm 1}Department of Computer Science, Czech Technical University, Prague, Czechia\\

    \textsuperscript{\rm 2} HSBC Lab, Innovation \& Ventures, HSBC Holdings, London, United Kingdom

}

\usepackage{bibentry}

\begin{document}

\maketitle

\begin{abstract}
There has been a great deal of recent interest in binarized neural networks, especially because of their explainability. At the same time, automatic differentiation algorithms such as backpropagation fail for binarized neural networks, which limits their applicability. 
We show that binarized neural networks admit a tame representation
by reformulating the problem of training binarized neural networks as a subadditive dual of a mixed-integer program, which we show to have nice properties. This makes it possible to use the framework of Bolte et al. for implicit differentiation, which offers the possibility for practical implementation of backpropagation in the context of binarized neural networks. 

This approach could also be used for a broader class of mixed-integer programs, beyond the training of binarized neural networks, as encountered in symbolic approaches to AI and beyond.
\end{abstract}

\section{Introduction}
There has been a great deal of recent interest in binarized neural networks (BNNs) \cite{NIPS2016_d8330f85,courbariaux2016binarized,yuan2021comprehensive}, due to their impressive statistical performance \cite[e.g.]{rastegari2016xnor}, the ease of distributing the computation \cite[e.g.]{NIPS2016_d8330f85}, and especially their explainability. This latter property, which is rather rarely encountered in other types of neural networks, stems precisely from the binary representation of the outputs of activation functions of the network, which can be seen as logical rules. 
This explainability is increasingly mandated by regulation of artificial intelligence, including the General Data Protection Regulation and the AI Act in the European Union, and the Blueprint for an AI Bill of Rights pioneered by the Office of Science and Technology Policy of the White House. 
The training of BNNs typically utilizes the Straight-Through-Estimator (STE) \cite{courbariaux2015binaryconnect,courbariaux2016binarized,rastegari2016xnor,zhou2016dorefa,lin2017towards,bulat2017binarized,cai2017deep,xiang2017binary}, 
where the weight updates in back-propagation unfortunately \cite{Alizadeh2018AnES} do not correspond to subgradients of the forward paths. 
This can lead to poor stationary points \cite{yin2019understanding}, and thus poor explanations. 

Here, we draw a new relationship between binarized neural networks and so-called \emph{tame geometry} \cite{van1998tame} to address this challenge. 
We introduce a certain reformulation of the training of BNN, 
which allows us to make use of the results of implicit differentiation  and non-smooth optimization when training the BNNs \cite{davis2020stochastic,bolte2021conservative,bolte2021nonsmooth,Bolte2022a} 
and, eventually, to obtain 
weight updates in the back-propagation that do correspond to subgradients of the forward paths 
in common software frameworks built around automated differentiation, such as TensorFlow or PyTorch.

This builds on a long history of work on \emph{tame topology} and \emph{o-minimal structures} \cite[e.g.]{grothendieck_1997,van1998tame,kurdyka1998gradients,kurdyka2000proof,fornasiero2008definably,fornasiero2010tame,kawakami2012locally,FORNASIERO2013211,fujita2023locally}, long studied in topology, logic, and functional analysis. 

Our reformulation proceeds as follows: 
In theory, the training of BNNs can be cast as a mixed-integer program (MIP). 
We formulate its sub-additive dual, wherein we leverage the insight that conic MIPs admit a strong dual in terms of non-decreasing subadditive functions. 
We show that this dual problem is tame, or definable in an o-minimal structure.
This, in turn, makes it possible for the use of powerful methods from non-smooth optimization when training the BNN, 
such as a certain generalized derivative of \cite{bolte2021conservative}
that comes equipped with a chain rule. 
Thus, one can use backpropagation,
as usual in training of nerural networks. 

In the process, we establish a broader class of \emph{nice} MIPs that admit such a tame reformulation. 
A MIP is nice if its feasible set is compact,
and the graph of the objective function has only a finite number of non-differentiable points.
This class could be of independent interest, as it may contain a number of other problems, such as
learning causal graphs \cite{chen2021integer}, optimal decision trees \cite{nemecek2023improving,nemecek2023piecewise}, or certain problems in symbolic regression \cite{austel2020symbolic,kim2023learning}.  
We hope that this could bring closer symbolic approaches, which can often be cast as MIPs,
and approaches based on neural networks and backpropagation.

\section{Background}
Let us start by introducing the relevant background material. We begin by introducing the relevant notions of BNNs, MIPs, and their subadditive dual. We discuss how the BNN can be recast as a MIP, and thus, by strong duality, how training the BNN relates to a maximization problem over a set of subadditive functions. Our main goal is to link the BNN with tame geometry, and therefore we discuss the relevant background on o-minimal structures. Finally, we discuss results on implicit differentiation for tame functions, which offers a practical way of training the BNN once we have established its tameness.

\paragraph{Binarized Neural Networks}
There is some ambiguity in the literature as to what constitutes a binarized neural network (BNN). We will follow \cite{bah2020integer} and refer to a BNN as a neural network where the activation functions take values in the binary set $\{0,1\}$. A BNN is characterized by a vector $L = (L_0,\ldots, L_n)$ with $|L|=n$ layers where each layer contains ${L_\ell} \in \mathbb{N}_{>0}$ neurons $x_i^{(\ell)}$, see Fig. \ref{fig:bnn}. We allow the input layer $x_i^{(0)}$ to take any real values, $x_i^{(0)}\in\BR$, while due to binarized activations, the following layers will have $x_i^{(j>0)}\in\{0,1\}$. The neuron $x_i^{(\ell)}$ in the layer $\ell$ is connected with the neuron $x_{j}^{(\ell+1)}$ in the layer $\ell+1$ via a weight coefficient matrix $w_{}^{(\ell)}\in \mathbb{R}^{L_{\ell}\times L_{\ell+1}}$. Consider an input vector $\boldsymbol{x} = (x_1^{(0)},\ldots, x_{L_0}^{(0)})$. The \emph{preactivation function} of the BNN is given as
\begin{align}
    a_j^{(\ell+1)}(\boldsymbol x) =  \sum_{i \in L_{\ell+1}} w_{ij}^{(\ell)}\sigma_j^{(\ell)}(\boldsymbol x),
\end{align}
where $\sigma^{(\ell)}(\boldsymbol x)$ is the \emph{activation} function at layer $\ell$ with 
\begin{align}\label{eq:activation}
    \sigma_j^{(\ell)}(\boldsymbol x) = \begin{cases}
       \boldsymbol{x} &\text{if } \ell=0, \\
       1 &\text{if } \ell>0 \text{ and } a^{(\ell)}_j(\boldsymbol x) \geq \lambda_{\ell},\\
       0 &\text{otherwise},
    \end{cases}
\end{align}
where $\lambda_{\ell}\in\BR$ is a learnable parameter. 
Note again that the activation functions of all the neurons in the network of our BNN are constrained in the set $\{ 0,1\}$ except for the input layer
neurons. This set can be mapped to $\{-1,1\}$ by a redefinition $\tilde{\sigma}_j^{(\ell)}=2\sigma_j^{(\ell)}-1$.

\begin{figure}[!tb]
    \centering
    \includegraphics[scale=0.8]{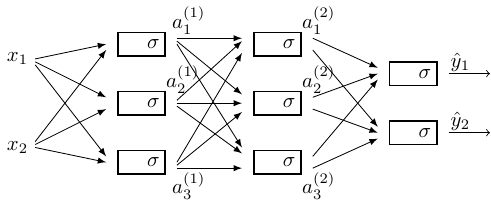}
    \caption{A BNN with $|L|=4$ layers, $L=(2,3,3,2)$. The input vector $\boldsymbol{x}=(x_1,x_2)$ take values in $\mathbb{R}^2$, while the activation functions, $\sigma$, in the following layers compress this to lie in the set $\{0,1\}^{L_\ell}$, for $\ell>0$.}
    \label{fig:bnn}
\end{figure}

The BNN can be viewed as a weight assignment $w = \{ w^{(1)},\ldots, w^{(L)} \}$ for a function
\begin{align}
    f_w : \mathbb{R}^{L_0} &\to \{0,1 \}^{L_n} \\
       \boldsymbol{x}^{(0)} &\mapsto \hat{\boldsymbol{y}},
\end{align}
where $\hat{\boldsymbol{y}} = \boldsymbol{x}^{(L_N)}$ is the vector of output layer neurons.
BNNs are trained by finding an optimal weight assignment $W$ that fits and generalizes a training set $S = \{ (\boldsymbol{x}_1,\boldsymbol{y}_1), \ldots, (\boldsymbol{x}_m, \boldsymbol{y}_m)\}$. The traditional approaches of backpropagation and gradient descent methods in usual deep learning architectures cannot be used directly for training BNNs. For optimizers to work as in standard neural network architectures, real-valued weights are required, so, in practice, when binarized weights and/or activation functions are utilized, one still uses real-valued weights for the optimization step. Another problem is related to the use of deterministic functions \eqref{eq:activation} or stochastic functions \cite{NIPS2016_d8330f85} for binarization, which ``flattens the gradient'' during backpropagation. 
A common solution to these problems is to use the Saturated STE (Straight Through Estimator) \cite{bengio2013estimating} (see also \cite{yin2019understanding}). Other possible solutions include the
Expectation BackPropagation (EBP) algorithm \cite{NIPS2014_076a0c97} which is a popular approach to training multilayer neural networks with discretized weights, and Quantized BackPropagation (QBP) \cite{hubara2018quantized}. Ref. \cite{Alizadeh2018AnES} presents a comprehensive practical survey on the training approaches for BNNs. 
In this article, we suggest that BNNs can be efficiently trained using nonsmooth implicit differentiation \cite{bolte2021nonsmooth}.




\paragraph{Mixed-Integer Programming} 
A mixed-integer linear program (MILP) is an optimization problem of the form
\begin{equation}\label{mip}
    \begin{aligned}
       \max & \quad c x+h y \\
       \rm{s.t.} & \quad A x+G y \geq b \\
                & \quad  x \in \mathbb{Z}_{\geq 0}  \\
                 & \quad  y \in \mathbb{R}_{\geq 0}.\
    \end{aligned}
\end{equation}
As illustrated in Figure~\ref{fig:mip}, the feasible set is a subset of the intersection of a polyhedron with the integral grid.

\begin{figure}[!tb]
    \centering
    \includegraphics[scale=0.65]{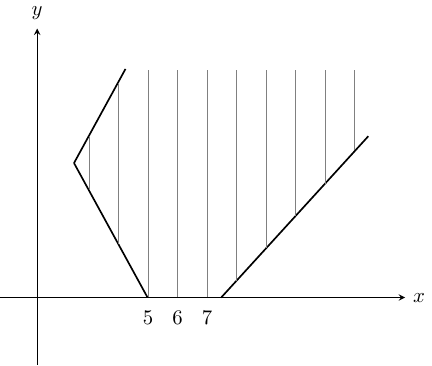}
    \caption{An illustrative example of a mixed-integer set as subset of $\mathbb{Z}\times \mathbb{R}$. This is given as the feasible set of the MIP \eqref{mip}.}
    \label{fig:mip}
\end{figure}

\paragraph{Recasting a BNN as a MIP}
The interactions and relations between BNNs and MILPs have been studied in recent literature. For example, in Ref. \cite{ToroIcarte2019} BNNs with weights restricted to $\{-1,1\}$ are trained by a hybrid method based on constraint programming and mixed-integer programming. Generally, BNNs with activation functions taking values in a binary set and with arbitrary weights can be reformulated as a MIP \cite{bah2020integer}. However, the precise form of the corresponding MIP depends on the nature of the loss function. Generally, a loss function for a BNN with $n$ layers is a map
    $\mathscr{L}:\{0,1\} \times \mathbb{R}^{L_n} \rightarrow \mathbb{R}$,
which allows the BNN to be represented as 
\begin{align}
\min & \sum_{i=1}^m \mathscr{L}\left(y_i, \hat{y}_i\right) \tag{BNN-MINLP} \label{eq:BNN} \\
\text { s.t. }  \hat{y}^i&=a^L\left(w^{(L)} a^{(L-1)}\left(\ldots a^{(1)}\left(w^{(1)} x_i\right) \ldots\right)\right) \notag \\
 w^{(\ell)} &\in \mathbb{R}^{L_\ell \times L_{\ell+1}}, \quad \forall \ell, \notag \\
 \lambda_\ell &\in \mathbb{R}, \quad \forall \ell, \notag \\ 
\hat{y}&\in \{0,1\}^m. \notag
\end{align}
The loss function $\mathscr{L}$ can be chosen in different ways; for example the 0-1 loss function $\mathscr{L}(\hat{y},y) = I_{\hat{y},y}$, where $I$ is the indicator function, or the square loss $\mathscr{L}(\hat{y},y) = \| \hat{y}-y \|^2$. The following result will then be essential for us \cite[Thm. 2]{bah2020integer}:

\begin{theorem}[MILP formulation]\label{thm:BNNMilp}
\eqref{eq:BNN} is equivalent to the following mixed-integer linear program:
\begin{align}\label{eq:linMIP}
\min & \quad \sum_{i=1}^m \mathscr{L}\left(y, u^{ (L)}\right) \tag{BNN-MILP} \\
\mathrm{s.t.}\quad w^{(1)} x^i & <M_1 u^{ (1)}+\lambda_1 \quad  \notag \\
\quad w^{(1)}  x^i &\geq M_1\left(u^{ (1)}-1\right)+\lambda_1  \notag \\ 
 \quad\sum_{l=1}^{d_{k-1}} s_l^{ (k)}& < M_k u^{ (k)}+\lambda_k, \quad \forall  k \in[L] \backslash\{1\} \notag \\ 
\quad \sum_{l=1}^{d_{k-1}} s_l^{(k)} &\geq M_k\left(u^{(k)}-1\right)+\lambda_k, \quad  \forall k \in[L] \backslash\{1\}\notag \\
\quad s_{l j}^{(k)} &\leq u_j^{ (k)}, \quad s_{l j}^{ (k)} \geq-u_j^{(k)},\notag \\
\quad \forall  k &\in[L] \backslash\{1\}, l \in\left[d_{k-1}\right], j \in\left[d_k\right]\notag \\
\quad s_{l j}^{(k)} &\leq w_{l j}^{(k)}+\left(1-u_j^{(k)}\right), \notag \\
&\quad \forall k \in[L] \backslash\{1\}, l \in\left[d_{k-1}\right], j \in\left[d_k\right] \notag \\
\quad s_{l j}^{(k)} & \geq w_{l j}^{(k)}-\left(1-u_j^{(k)}\right), \notag \\
\quad\forall  k &\in[L] \backslash\{1\}, l \in\left[d_{k-1}\right], j \in\left[d_k\right] \notag \\
 \quad W^k &\in[-1,1]^{d_k \times d_{k-1}}\quad  \forall k \in[L] \notag \\
\quad \lambda_k &\in[-1,1] \quad \forall k \in[L] \notag \\
\quad u^{i, k} &\in\{0,1\}^{d_k} \quad  \forall k \in[L], i \in[m] \notag \\
\quad s_l^{i,k} &\in [ -1,1]^{d_k} \quad \forall i \in[m], k \in[L] \backslash\{1\}, l \in\left[d_{k-1}\right], \notag
\end{align}
where $x\coloneqq x^{(0)}$, $u^{(\ell)}\coloneqq x^{(\ell)}$, for $0<\ell\leq L$, $M_1:=(n r+1)$, $\| x\|<r$ a Euclidean norm bound, $n$ the dimension of $x$, and $M_\ell:=\left(d_{\ell-1}+1\right)$. The new variables $s^{(\ell)}_{ij}\in[-1,1]$ have been added to linearize the products $w^{(\ell)}x^{(\ell-1)}$ that would otherwise appear. Finally, we have rescaled the weights $w$ and parameters $\lambda$ to lie in $[-1,1]$, without loss of generality. See also Lemma 1 in \cite{bah2020integer}.
\end{theorem}

This gives the first step in our aim to link the theory of tame geometry, or o-minimality, to BNNs. The next step is to look at the dual problem of this MILP. 

\paragraph{Subadditive dual}

In the context of MIPs, the notion of duality is much more involved than in convex optimization 
\cite{guzelsoy2010integer}. Only recently \cite{Kocuk2019,moran2012strong}, it is emerging that subadditive duals  \cite{Jeroslow1978,Johnson1980,Johnson1974,Johnson1979,Guzelsoy2007DualityFM,Wolsey1999-ez} can be used to establish strong duality for MIPs.
To introduce the subadditive dual, we use the modern language of \cite{moran2012strong,Kocuk2019}:

\begin{definition}[Regular cone]
A cone $K\subseteq \BR^m$ is called regular if it is closed, convex, pointed and full-dimensional. 
\end{definition}
If $x-y\in K$, we write $x\succeq_K y$ and similarily, if $x\in \text{int}(K)$ we write $x\succ_K 0$.

\begin{definition}[Subadditive and non-decreasing functions]
    A function $f:\BR^m\to\BR$ is called:\begin{itemize}
    \item \emph{subadditive} if $f(x+y)\leq f(x)+f(y)$ for all $x,y\in \BR^m$;
    \item \emph{non-decreasing} with respect to a regular cone $K\subseteq \BR^m$ if $x\succeq_K y\implies f(x)\geq f(y)$. 
    \end{itemize}
\end{definition}
The set of subadditive functions that are non-decreasing with respect to a regular cone $K \subseteq \BR^m$ is denoted $\mathcal{F}_K$ and for $f\in\mathcal{F}_K$ we further define $\bar f(x)\coloneqq \limsup_{\delta\to 0^+}\tfrac{f(\delta x)}{\delta}$. Note that this is the upper $x$-directional derivative of $f$ at zero. 

Let us start by stating the relation between subadditive functions and MIPs. To this end, we consider a generic conic MIP,
\begin{equation}\label{eq:conic_milp}
    \begin{aligned}
        z^*\coloneqq &\inf c^T x+d^Ty,\\
        \text{s.t. }&Ax+Gy\succeq_K b,\\
        &x\in\mathbb{Z}^{n_1},\\
        &y\in\BR^{n_2}.
    \end{aligned}
\end{equation}
Note that problem \eqref{eq:conic_milp} is a generalization of the primal form of a MILP, as in Thm. \ref{thm:BNNMilp}, which is recovered by setting $K=\mathbb{R}^{m}_+$. 
We define the subadditive dual problem of \eqref{eq:conic_milp} as
\begin{equation}\label{eq:subadditive}
    \begin{aligned}
        \rho^*\coloneqq &\sup f(b),\\
        \text{s.t. }&f(A^j)=-f(-A^j)=c_j,\quad j=1,\dots,n_1,\\
        &\bar{f}(G^k)=-\bar{f}(-G^k)=d_k,\quad k=1,\dots,n_2,\\
        &f(0)=0,\\
        &f\in\mathcal{F}_K, 
    \end{aligned}
\end{equation}
where $A^j$ and $G^j$ denotes the $j$'th column of the matrices $A$ and $G$, respectively, and $c_j,d_k$ are the components of the corresponding vectors from the primal MIP.

In general, the subadditive dual \eqref{eq:subadditive} is a weak dual to the primal conic
MIP \eqref{eq:conic_milp}, where any dual feasible solution provides a lower bound for the optimal value
of the primal \cite{Zlinescu2011,BenTal2001,moran2012strong}. Under the assumptions of feasibility, strong duality holds:

\begin{theorem}[Thm. 3 of \cite{Kocuk2019}]\label{thm:sufficient}
    If the primal conic MIP \eqref{eq:conic_milp} and the subadditive dual \eqref{eq:subadditive} are both feasible, then \eqref{eq:subadditive} is a strong dual of \eqref{eq:conic_milp}. Furthermore, if the primal problem is feasible, then the subadditive dual is feasible if and only if the conic dual of the continuous relaxation of \eqref{eq:conic_milp} is feasible.
\end{theorem}

That is: Theorem \ref{thm:sufficient} provides a
sufficient condition for the subadditive dual to be equivalent to \eqref{eq:conic_milp}. A sufficient condition for the dual feasibility is that the conic MIP has a bounded feasible region.

\paragraph{Properties of subadditive functions}
To show our main result, Theorem \ref{th:main}, we will need to introduce some structural properties of subadditive functions. These are discussed in detail in \cite{Rosenbaum1950,Matkowski1993Subadditive,Bingham2008GenericSF}. For example, if $f,g$ are two non-decreasing subadditive functions on $\BR^m$, then the following hold:
\begin{itemize}
    \item $f+g$ is subadditive;
    \item the composition $g\circ f$ is subadditive;
    \item if further $f$ is non-negative and $g$ positive on the positive quadrant $\BR_+^m$ then $f(x)g(x)$ is subadditive on $\BR_+^m$. 
\end{itemize}
Let us note that, when we set $K=\BR_+^m$ in \eqref{eq:subadditive} we have that $f(x)$ is non-negative on $\BR_+^m$ due to the combination of being non-decreasing, subadditive and having the condition $f(0)=0$. 

Following \cite{Bingham2008GenericSF}, we define properties NT (as in ``no trumps'') and WNT (for ``weak no trumps''), see also Def. 1 and 2 of \cite{Bingham2008GenericSF}:

\begin{definition}[NT]
    For a family $(A_k)_{k\in\BN}$ of subsets of $\BR^n$ we say that $\textbf{NT}(A_k)$ holds, if for every bounded/convergent sequence $\{a_j\}$ in $\BR^n$ some $A_k$ contains a translate of a subsequence of $\{a_j\}$. 
\end{definition}

\begin{definition}[\textbf{WNT}]
Let $f\,:\,\BR^n\to \BR$. We call $f$ a \textbf{WNT}-function, or $f\in\textbf{WNT}$, if $\textbf{NT}(\{F^j\}_{j\in\BN^*})$ holds, where $F^j\coloneqq \{x\in \BR^n\,:\,|f(x)|<j\}$.
\end{definition}

We have the following theorem by Csisz\'ar and Erd\"os \cite{csiszar1964lim}, nicely explained in \cite{bingham2009beyond}: 

\begin{theorem}[NT theorem, \cite{csiszar1964lim}]
\label{thm:notrumps}
    If $T$ is an interval and $T=\bigcup_{j\in\BN^*}T_j$ with each $T_j$ measurable/Baire, then $\textbf{NT}(\{T_k\,:\, k\in\BN^*\})$ holds. 
\end{theorem}

Here, Baire refers to the functions having ``the Baire property'', or the set being open modulo some meager set. Note that this is not necessarily related to being definably Baire as in Def. \ref{def:baire2}. 

The following properties are shown in \cite{Bingham2008GenericSF}:
\begin{itemize}
    \item If $f$ is subadditive and locally bounded above at a point, then it is locally bounded at every point. 
    \item If $f\in \textbf{WNT}$ is subadditive, then it is locally bounded. 
    \item If $f\in \textbf{WNT}$ is subadditive and $\inf_{t<0}f(tx)/t$ is finite for all $x$, then $f$ is Lipschitz. 
\end{itemize}

\paragraph{Tame topology}\label{sec:tame_top}
The subject of tame topology goes back to Grothendieck and his famous ``Esquisse d'un programme'' \cite{grothendieck_1997}. Grothendieck claimed that modern topology was riddled with false problems, which he ascribed to the fact that much of modern progress had been made by analysts. What he proposed was the invention of a geometers version of topology, lacking these artificial problems from the onset. Subsequently, tame topology has been linked to model-theoretic notions of o-minimal structures, which promise to be good candidates for Grothendieck's dream. 
o-minimal structures are a generalisation of the (semi-)algebraic sets, or the sets of polynomial equations (and inequalities). As such, they provide us with a large class of sets and functions that are in general non-smooth and non-convex, while capturing most (if not all) of the popular settings used in modern neural networks and machine learning \cite{davis2020stochastic}. 

An o-minimal structure over $\real$ is a collection of subsets of $\real^m$ that satisfies certain finiteness properties, such as closure under boolean operations, closure under projections and fibrations. Formally,
\begin{definition}[o-minimal structure]
	    An o-minimal structure on $\BR$ is a sequence $\CS=(\CS_m)_{m\in\BN}$ such that for each $m\geq 1$:
	    \begin{enumerate}\setlength\itemsep{0.5em}
	        \item[{1)}] $\CS_m$ is a boolean algebra of subsets of $\BR^m$;
	        \item[{2)}] if $A\in\CS_m$, then $\BR\times A$ and $A\times \BR$ belongs to $\CS_{m+1}$;
	        \item[{3)}] $\CS_m$ contains all diagonals, for example $\{(x_1,\dots,x_m)\in \BR^m\,:\, x_1=x_m\}\in \CS_m$;
	        \item[{4)}] if $A\in\CS_{m+1}$, then $\pi(A)\in\CS_m$;
         \item[{5)}] the sets in $\CS_1$ are exactly the finite unions of intervals and points. 
	    \end{enumerate}
	\end{definition}
Typically, we refer to a set included in an o-minimal structure as being \emph{definable} in that structure, and similarly, a function, $f:\real^m\to\real^n$, is called definable in an o-minimal structure whenever its corresponding graph, $\Gamma(f)=\{(x,y)\,|\, f(x)=y\}\subseteq \real^{m\times n}$, is definable. A set, or function, is called \emph{tame} to indicate that it is definable in some o-minimal structure, without specific reference to which structure.  

The moderate sounding definition of o-minimal structures turns out to include many non-trivial examples. First of all, by construction, semialgebraic sets form an o-minimal structure, denoted $\Ralg$. If this was the only example of an o-minimal structure, it would not have been a very interesting construction. The research in o-minimal structures really took off in the middle of the nineties, after Wilkie \cite{Wilkie1996ModelCR} proved that we can add the graph of the real exponential function, $x\mapsto e^x$, to $\Ralg$ to again find an o-minimal structure, denoted $\Rexp$. As a result, the sigmoid function, which is a prevalent activation function in numerous neural networks, can be considered tame. Another important structure is found by including the set of restricted real-analytic functions, where the domain of an analytic function is restricted to lie in a finite subset of the original domain in a particular way. This gives rise to an o-minimal structure denoted $\Ran$ \cite{vdDRanexp}. A classical example of this would be the function $\sin(x)$, where we restrict $x$ to lie in a finite interval $x\in [0,\alpha]\subset \real$ for some $\alpha<\infty$. Note that without this restriction on the argument, $\sin(x)$ is not tame. Furthermore, we can construct a very important o-minimal structure by combining $\Rexp$ with $\Ran$. This gives the structure denoted $\Ranexp$ \cite{vdDRanexp}. It is important to note that the fact that $\Ranexp$ is an o-minimal structure is a non-trivial result. In general, it does not hold that the combination of two o-minimal structures gives another o-minimal structure. 

We thus see that o-minimal structures capture a very large class of, generally, non-smooth non-convex functions. More importantly, they include all classes of functions widely used in modern machine learning applications. The great benefit of this class is that they are still \emph{nice} enough such that we can have some control over their behaviour and prove convergence to optimal points \cite{Bolte2009Semismooth,davis2020stochastic,bolte2021conservative,aravanis2022polynomial,josz2023global}.

Perhaps the most fundamental results regarding o-minimal structures are the monotonicity and cell decomposition theorems. The former states that any tame function of one variable can be divided into a \emph{finite} union of open intervals, and points, such that it is continuous and either constant or strictly monotone on each interval. The cell decomposition theorem generalizes this to higher dimensions by introducing the concept of a cell, which is the analogue of the interval or point in one dimension. The theorem then states that any tame function or set can be decomposed into a finite union of definable cells. A related notion is that of a stratification of a set. Generally, a stratification is a way of partitioning a set into a collection of submanifolds called strata. There exist many different types of stratifications, characterized by how the different strata are joined together. Two important such conditions are given by the Whitney and Verdier stratifications. Both of these are applicable to tame sets \cite{loi1996whitney,le1998verdier}. These results are at the core of many of the strong results on tame functions in non-smooth optimization.


\paragraph{Locally o-minimal structures}
There exists a few variants of weakenings of the o-minimal structures. One such example is what is called a \emph{locally o-minimal structure} \cite{fornasiero2008definably,fornasiero2010tame,kawakami2012locally,FORNASIERO2013211,fujita2023locally}.

\begin{definition}[Locally o-minimal structure]
    A definably complete structure $\mathbb{K}$ extending an ordered field is \emph{locally o-minimal} if, for every definable function $f\,:\,\mathbb{K}\to\mathbb{K}$, the sign of $f$ is eventually constant. 
\end{definition}
Here, definably complete means that every definable subset of $\mathbb{K}$ has a supremum in $\mathbb{K}\sqcup \{\pm\infty\}$ and $X\subseteq\mathbb{K}$ is nowhere dense if $\text{Int}(\bar X)$ is empty. Every o-minimal expansion of an ordered field is a definably complete structure (but the converse is not true). Note also that every o-minimal structure is locally o-minimal \cite{fornasiero2010tame}. Locally o-minimal structures satisfy a property called \emph{definably Baire}:
\begin{definition}[Definably Baire, from \cite{FORNASIERO2013211}]\label{def:baire2}
    A definably complete structure $\mathbb{K}$ expanding an ordered field is \emph{definably Baire} if $\mathbb{K}$ is not the union of a definable increasing family of nowhere dense subsets. 
\end{definition}

Finally, when we work with structures expanding $(\BR,+,\cdot,<)$ we have that local o-minimality implies o-minimality \cite{kawakami2012locally}, while the same is generically not true when we do not have multiplication.

\paragraph{Non-smooth differentiation}\label{sec:conservative} Bolte et al., \cite{bolte2021conservative}, introduced a generalized derivative, called a conservative set-valued field, for non-smooth functions. The main idea behind this construction is that the conservative fields come equipped with a chain rule. Namely, given a locally Lipschitz function $f:\,\BR^m\to\BR$, we say that $D:\,\BR^m\rightrightarrows \BR^m$ is a conservative field for $f$ if and only if the function $t\mapsto f(x(t))$ satisfies 
\begin{equation}
    \frac{\rd }{\rd t}f(x(t))=\langle v,\dot x(t)\rangle,\quad \forall v\in D(x(t)),
\end{equation}
for any absolutely continuous curve $x:\,[0,1]\to\BR^m$ and for almost all $t\in[0,1]$.
Having a chain rule is key for applications to backpropagation algorithms and automatic differentiation in machine learning. 

Automatic differentiation for non-smooth elementary functions is subtle and even the well-known Clarke generalized gradient is known to introduce complications in this setting. Having a derivative flexible enough to include automatic differentiation was therefore indeed the main motivation behind the work of Bolte et al. In many ways, we can see the conservative fields as a generalization of the Clarke derivatives. 

The conservative fields provide a flexible calculus for non-smooth differentiation that is applicable to many machine learning situations. In \cite{bolte2021nonsmooth}, a non-smooth implicit differentiation using the conservative Jacobians is developed. This can be seen as a form of automatic subdifferentiation (backpropagation). The automatic subdifferentiation is an automated application of the chain rule, made available through the use of the conservative fields. It amounts to calculating the conservative Jacobians of the underlying functions. This "conservative subgradient descent" is given by picking an initial value for the parameters, captured by a vector $v_0$ followed by performing the following update in steps
\begin{equation}
    \begin{aligned}
        v_{k+1}=v_k+\alpha_kg_k,\\
        g_k\in J(v_k),
    \end{aligned}
\end{equation}
with $(\alpha_k)_{k\in\BN}$ a sequences of step-sizes and $J(v_k)$ the conservative Jacobian \cite{bolte2021nonsmooth}. 

This gives a formal mathematical model for propagating derivatives which can be applied to guarantee local convergence of mini-batch stochastic gradient descent with backpropagation for a large number of machine learning problems. In particular, and of great importance for us, these results hold for locally Lipschitz tame functions.

Next, we will show that the subadditive dual of the MIP formulation of the BNN \eqref{eq:BNN} is locally Lipschitz and tame. This will allow us to use the machinery of \cite{bolte2021conservative,bolte2021nonsmooth} discussed above when training the BNN.

\section{Main Result}
We will now present the main result of the paper. To do so, we will restrict to a certain subset of conic MIPs, which we call \emph{nice}: 

\begin{definition}\label{def:nice}
Let us consider a conic MIP \eqref{eq:conic_milp}. 
    Under the following conditions:
\begin{enumerate}
\item[(A)] the conic MIP is feasible, 
\item[(B)] the conic dual of the continuous relaxation of \eqref{eq:conic_milp} is feasible, 
\item[(C)] the graph of the objective function has a finite number of non-differentiable points,
\end{enumerate}
we call the conic MIP \emph{nice}. 
\end{definition}

For example:

\begin{proposition}
The conic MIP of Theorem \ref{thm:BNNMilp} is nice.
\label{thm:main2}
\end{proposition}

\begin{proof}
    The feasible set is a product of $\{0, 1\}^{L_nm}
$ and the set $S$. For any value in $\{0, 1\}^{L_nm}$, we obtain a finite value within $S$.
    The feasible set is then compact.  
    Theorem 5 of \cite{Kocuk2019} then tells us that condition (B) of Definition \ref{def:nice} is satisfied. 
The objective function is a finite sum of loss functions for the original BNN, and as such it has a finite number of non-differentiable points, satisfying condition (C). 
\end{proof}

\begin{theorem}\label{th:main}
For a \emph{nice} conic MIP  \eqref{eq:conic_milp}, there exists an equivalent reformulation that is definable in an o-minimal structure.
\label{thm:main}
\end{theorem}

\begin{proof}
Let us consider the subadditive dual \eqref{eq:subadditive} of the nice conic MIP \eqref{eq:conic_milp}. 
When the conic dual of the continuous relaxation is feasible, this dual is equivalent by Theorem \ref{thm:sufficient}.
Furthermore, this dual is locally o-minimal by considering the No-Trumps theorem (Theorem \ref{thm:notrumps}) together with the fact that $f(x)$ is non-decreasing and subadditive.
By \cite[Remark 22]{kawakami2012locally}, a compact subset of a locally o-minimal
structure is o-minimal. 
When we consider that the continuous relaxation of the mixed-integer set is bounded (cf.\ Property (B) of Definition \ref{def:nice} together with Thm. 5 of \cite{Kocuk2019}), we thus obtain o-minimality.
\end{proof}

\begin{corollary}
Training BNNs allows for implicit differentiation and chain rule. 
\end{corollary}

\begin{proof}
This follows from Proposition \ref{thm:main2} and Theorem \ref{thm:main} together with the work of \cite{bolte2021conservative, bolte2021nonsmooth} discussed above, when one realizes that the subadditive dual is locally Lipschitz.
Lipschitzianity is from \cite{Bingham2008GenericSF}:
If $f \in \textbf{WNT}$ is subadditive and $\inf_{t<0}f(tx)/t$ is finite for all $x$, then $f$ is Lipschitz. 
\end{proof}

This corollary thus provides us with a practical way of training the BNNs, by utilizing the results of \cite{bolte2021conservative,bolte2021nonsmooth} to optimize over the subadditive dual of the corresponding MILP.

Let us finally note that, in general, non-decreasing subadditive functions are not tame. A counterexample is given by the Cantor staircase function \cite{dobovs1996standard}. This means that in general, the subadditive dual of a conic MIP need not fall under the tame setting and some additional property (``constraint qualification'') is necesseary for our main result.

\section{An Example}
\label{sec:example1} 

To make the above discussion more clear, we present a simple example outlining how the training of a BNN could make use of the implicit differentiation  \cite{bolte2021nonsmooth}. To this end, we consider three final layers of a BNN inspired by Example 1 of \cite{Guzelsoy2007DualityFM},
illustrated in Fig. \ref{fig:ralphs_ExampleBNN}, 
where there are a number of binary weights given by the final layer of \eqref{eq:linMIP}, $u_i^{(L)}$, $i=1,\dots,n$, to be learned. Here, of course $u^{(L)}=\hat{\boldsymbol{y}}=\boldsymbol{x}^{(L_N)}=a^{(L)}(w^{(L)}a^{(L-1)}(\dots a^{(1)}(w^{(1)}x)\dots))$. We split this vector into two, by introducing an $m\in\mathbb{N}$ such that $1<m<n$. The pen-ultimate two layers yield a bi-variate continuous-valued output layer $(Y_1, Y_2)$. Instead of the usual empirical risk, we consider an objective function involving weighted difference from values of the dependent variable in the training data (assumed to be zero), as well as one of the weights in the pen-ultimate layer, for the sake of a more interesting illustration:  
\begin{equation}\label{eq:ralphsPrimal}
    \begin{aligned}
        \min_{X,Y} \; & 2(Y_1 - 0) + (Y_2 - 0) + \tfrac12 X_1,\\
        \text{s.t. } &X_1-\tfrac32 X_2+Y_1-Y_2=b,\\
        & X_1=\sum_{i=1}^m u_i^{(L)}, \\
        & X_2=\sum_{i=m+1}^n u_i^{(L)}, \\
        &
        u_i^{(L)} \in \{ 0, 1\}, 
        \,
        X_1,X_2\in \mathbb{Z}_+,\, Y_1,Y_2\in\mathbb{R}_+.
    \end{aligned}
\end{equation}
We note that our theory does cover the case of the usual empirical risk with square-loss function, but the illustrations would be more involved due to the non-linearity in the square loss.

\begin{figure}[!tbh]
    \centering
    \includegraphics[scale=0.7]{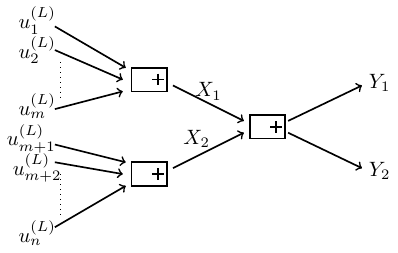}
    \caption{A graphical representation of the three final layers of the BNN we use as an example. See  \eqref{eq:ralphsPrimal} for the corresponding MILP.}
    \label{fig:ralphs_ExampleBNN}
\end{figure}

Following the definition 
\eqref{eq:subadditive}, the subadditive dual of \eqref{eq:ralphsPrimal} is:
\begin{equation}\label{eq:ralphsSubadditive}
    \begin{aligned}
        \max_{f\in \mathcal{F}_{\mathbb{R}_+}}\, &f(b),\\
        \text{s.t. } f(1)&\leq \tfrac12,\\
        f(-\tfrac32)&\leq 0,\\
        \bar f(1)&\leq 2,\\
        \bar f(-1)&\leq 1,\\
        f(0)&=0.
    \end{aligned}
\end{equation}
The subadditive dual problem is obviously an infinite-dimensional optimization problem over the whole space of subadditive functions $ \mathcal{F}_{\mathbb{R}_+}$. However, as shown in \cite{schrijver1980cutting}, the subadditive dual functions of MILPs are Chv{\' a}tal functions, i.e., piecewise-linear. We can thus utilize this knowledge to finitely parametrize the space of relevant subadditive functions by the number of segments, slopes, and breakpoints of  piecewise-linear subadditive functions. When we consider nice MILPs as in \eqref{def:nice}, we thus obtain a finite-dimensional problem. It is furthermore evident that we can approximate this problem by truncating in the number of segments of the piecewise-linear subadditive functions. 

For the above example, we start with approximating $f$ by a piecewise-linear function having two segments. By visual inspection of the behaviour of the value function $f(b)$ (in solid lines) near the origin in  Fig. \ref{fig:example}, we see that we can approximate $f(b)$ by
\begin{equation}\label{eq:approx}
    \tilde f(b)\coloneqq \begin{cases}
        2b,\quad &b> 0,\\
        -b,\quad &b\leq 0,
    \end{cases}
\end{equation}
based on the directional derivatives. This crude approximation  is shown in Fig. \ref{fig:example} as dashed lines. A conservative field for this function is given by
\begin{equation}
    D_{\tilde f}(b)=\begin{cases}
    2,\quad &b>0,\\
    [-1,2],\quad &b=0,\\
    -1,\quad &b<0.
    \end{cases}
\end{equation}
It is now clear that we can use the conservative fields of \cite{bolte2021conservative,bolte2021nonsmooth} to train over this approximation of the the piecewise-linear subadditive dual of the primal problem \eqref{eq:ralphsPrimal}.

More generally, we can introduce slope variables $s_1$ and $s_2$, as well as a breaking point $p$, to parametrize the two-segment approximation: 
\begin{equation}
    \tilde f(b)\coloneqq \begin{cases}
        &s_1b,\quad b>p,\\
        &s_2b,\quad b\leq p,
    \end{cases}
\end{equation}
and thus find the best two-segment approximation of the  piecewise-linear subadditive dual of the primal problem \eqref{eq:ralphsPrimal}, which in this case coincides with \eqref{eq:approx} above.  Next, we can increase the precision of the approximation by introducing more and more segments of this approximating function,
and optimize over the slopes and break points of the segments, and possibly also the number of segments.
Following \cite{bertsimas2017optimal}, we have studied \cite{nemecek2023piecewise} formulations based on the optimal regression trees for piecewise regression. 

\begin{figure}[!tb]
    \centering
    \includegraphics[scale=0.45]{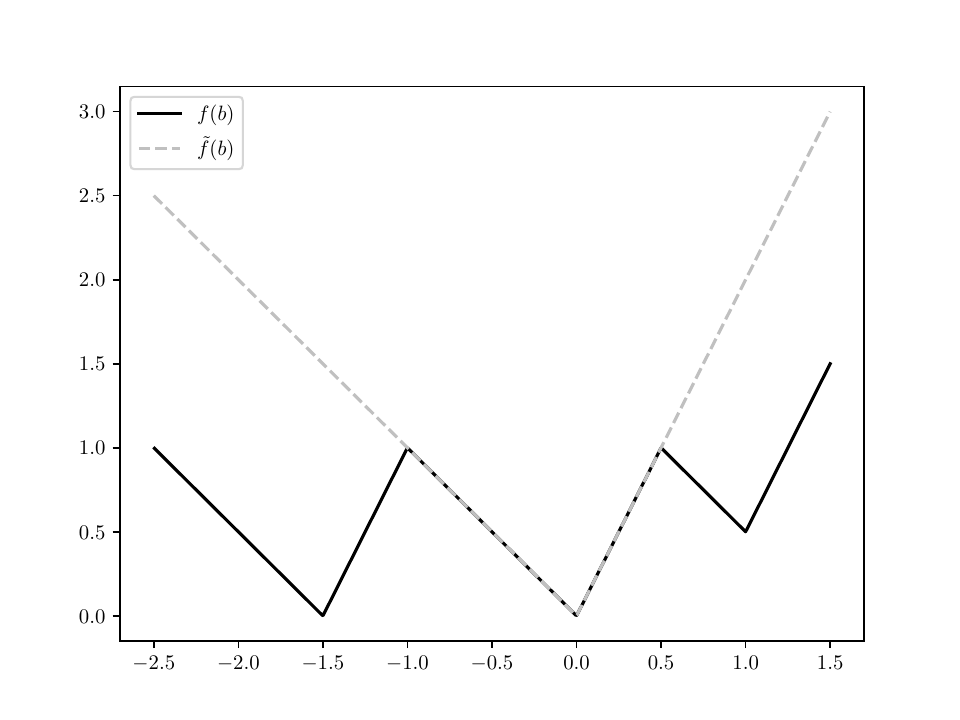}
    \caption{The value function, $f(b)$, of the example \eqref{eq:ralphsSubadditive} together with the simple approximation, $\tilde f(b)$, given by \eqref{eq:approx}.}
    \label{fig:example}
\end{figure}

\section{Conclusions and Limitations}
\label{sec:conclusions}

We have introduced a link between binarized neural networks, and more broadly, nice conic MIPs, and tame geometry. This makes it possible to reuse pre-existing theory and practical implementations of automatic differentiation.  
Breaking new ground, we leave many questions open. The foremost question is related to the efficiency of algorithms for constructing the subadditive dual. Although Guzelsoy and Ralphs \cite[Section 4, Constructing Dual Functions]{Guzelsoy2007DualityFM} survey seven very different algorithms, their computational complexity and relative merits are not well understood. For any of those, an efficient implementation (in the sense of output-sensitive algorithm) would provide a solid foundation for further empirical experiments.
Given the immense number of problems in symbolic AI, which can be cast as MIPs, and the excellent scalability of existing frameworks based on automatic differentiation, the importance of these questions cannot be understated. 

\paragraph{Acknowledgements}
 The research of Johannes Aspman has been supported by European Union’s Horizon Europe research and innovation programme under grant agreement No. GA 101070568 (Human-compatible AI with guarantees).
 This work received funding from the National Centre for Energy II (TN02000025).
 Jakub would like to acknowledge most pleasant discussions with Jerome Bolte in Prague, the Czech Republic, and with Ted Ralphs in Monterrey, Mexico.

\paragraph{Disclaimer} This paper was prepared for information purposes
and is not a product of HSBC Bank Plc. or its affiliates.
Neither HSBC Bank Plc. nor any of its affiliates make
any explicit or implied representation or warranty and
none of them accept any liability in connection with
this paper, including, but limited to, the completeness,
accuracy, reliability of information contained herein and
the potential legal, compliance, tax or accounting effects
thereof.

\clearpage
\bibliography{aaai24}

\end{document}